\newif\ifdraft
 \draftfalse

\documentclass[11pt]{article}

\usepackage[letterpaper,margin=1in]{geometry}
\usepackage[parfill]{parskip}

\usepackage{authblk}

\usepackage[toc,page,header]{appendix}

\usepackage[utf8]{inputenc} 
\usepackage[T1]{fontenc}    
\usepackage[colorlinks,citecolor=blue,urlcolor=blue,linkcolor=blue,linktocpage=true]{hyperref}       
\usepackage{url}            
\usepackage{booktabs}       
\usepackage{amsfonts}       
\usepackage{nicefrac}       
\usepackage{microtype}      
\usepackage{xcolor}         
\label{key}
\usepackage{thmtools}

\usepackage{comment}
\usepackage{amsmath}

\usepackage{epsfig}
\usepackage{graphicx}
\usepackage{wrapfig}
\usepackage{subfigure}
\usepackage{multirow}
\usepackage{hyperref}
\usepackage{graphicx}
\usepackage{caption}
\usepackage{array}
\usepackage{bbm}
\usepackage{color}
\usepackage{enumerate}
\usepackage{enumitem}
\setlist{leftmargin=10mm}
\usepackage{mathtools}
\usepackage{amsmath,amssymb,amsthm,bm}
\usepackage{amsfonts,graphicx}
\usepackage{mathrsfs}
\usepackage[ruled,vlined,linesnumbered]{algorithm2e}
\usepackage{algpseudocode}

\usepackage[authoryear]{natbib}

\newif\iffinal

\iffinal
    \newcommand{\tianhao}[1]{}
    \newcommand{\ruoxi}[1]{}
    \newcommand{\add}[1]{}
\else
    \newcommand{\tianhao}[1]{{\bf \textcolor{purple}{[Tianhao: #1]}}}
    \newcommand{\ruoxi}[1]{{\bf \textcolor{BrickRed}{[Ruoxi:#1]}}}
    \newcommand{\add}[1]{\textcolor{purple}{#1}}
\fi

\usepackage{cleveref}
\usepackage{dsfont}
\newtheorem{theorem}{Theorem}
\newtheorem{lemma}[theorem]{Lemma}

\newtheorem{definition}[theorem]{Definition}

\newtheorem*{remark}{Remark}

\newtheorem{remark-star}{Remark}
\newtheorem{remark-star-1}{Remark}

\newtheorem*{proof-sketch}{Proof Sketch}
\usepackage{booktabs}

\newtheorem{innercustomthm}{Theorem}
\newenvironment{customthm}[1]
  {\renewcommand\theinnercustomthm{#1}\innercustomthm}
  {\endinnercustomthm}

\newcommand{\E}{\mathbb{E}}

\newcommand{\R}{\mathbb{R}}
\newcommand{\eps}{\varepsilon}
\newcommand{\D}{\mathcal{D}}
\newcommand{\norm}[1]{\left\lVert#1\right\rVert}

\newcommand{\I}{\mathcal{I}}

\newcommand{\ind}{\mathds{1}}

\newcommand{\mS}{\mathcal{S}}
\newcommand{\own}{\ni}
\newcommand{\notown}{\not\own}

\newcommand{\A}{\mathcal{A}}
\newcommand{\metric}{\texttt{acc}}

\newcommand{\unif}{\mathrm{Unif}}

\newcommand{\Sowni}{\mS_{\own i}}
\newcommand{\Snotowni}{\mS_{\notown i}}

\newcommand{\tildeO}{\widetilde{O}}

\newcommand{\Utot}{U_{\mathrm{tot}}}
\newcommand{\qtot}{q_{\mathrm{tot}}}

\author[1]{Jiachen T. Wang}
\author[2]{Ruoxi Jia}

\affil[1]{Princeton University}
\affil[2]{Virginia Tech\protect\\
\texttt{\small tianhaowang@princeton.edu}, 
\texttt{\small ruoxijia@vt.edu}
}

\date{}

\title{A Note on ``Towards Efficient Data Valuation Based on the Shapley Value''}

\begin{document}



\maketitle

\begin{abstract}
The Shapley value (SV) has emerged as a promising method for data valuation. However, computing or estimating the SV is often computationally expensive. To overcome this challenge, \cite{jia2019towards} propose an advanced SV estimation algorithm called ``Group Testing-based SV estimator'' which achieves favorable asymptotic sample complexity. 
In this technical note, we present several improvements in the analysis and design choices of this SV estimator. Moreover, we point out that the Group Testing-based SV estimator does not fully reuse the collected samples. Our analysis and insights contribute to a better understanding of the challenges in developing efficient SV estimation algorithms for data valuation.

\end{abstract}

\section{Introduction}
Data valuation, i.e., measuring the contribution of a data source to the ML training process, is an important problem in the field of machine learning (ML). For example, assessing the value of data helps to identify and remove low-quality data \citep{ghorbani2019data, kwon2021beta}, and also provides insights into a model's test-time behavior \citep{koh2017understanding}. 
Additionally, data valuation plays a critical role in incentivizing data sharing and shaping policies for data market \citep{zhu2019incentive, tian2022data}. 

Cooperative game theory and economic principles have inspired the use of the Shapley value (SV) as a principled approach for data valuation \citep{ghorbani2019data, jia2019towards}. The SV is the unique notion that satisfies natural fairness requirements in the ML context. 
The SV has shown superior performance on many ML tasks such as identifying mislabeled data or outliers. 

Despite these advantages, the SV is known to be computationally expensive. The number of utility function evaluations required by the exact SV calculation grows exponentially in the number of players (i.e., data points for the ML context). 
Even worse, for ML tasks, evaluating the utility function itself (e.g., computing the testing accuracy of the ML model trained on a given dataset) is already computationally expensive, as it requires training a model. 
For a target precision, the classic permutation sampling SV estimator \citep{castro2009polynomial} requires evaluating the utility functions for $\Omega(N^2 \log N)$ times for estimating the SV for $N$ data points. 
\cite{jia2019towards} propose an advanced estimation algorithm termed as ``Group Testing-based SV estimator'', which reduces the sample complexity to $\Omega(N (\log N)^2)$. 
This algorithm achieves greater asymptotic efficiency by increasing sample reuse, allowing each evaluation of utility to be used in the estimation of the SV for \emph{all} $N$ data points. 

In this note, we present several improvements for the analysis and algorithm design of the Group Testing-based SV estimator. Specifically, we give a more fine-grained and simpler sample complexity analysis for the Group Testing-based SV estimator in \cite{jia2019towards}. Moreover, we propose two modifications for the estimation algorithm. Both the new analysis and the new algorithm design save constant factors for the sample complexity of SV estimators. Such an improvement is significant when the non-asymptotic sample bound is used for providing confidence intervals for the estimation. 


Additionally, we point out that 
while the Group Testing-based SV estimator seems to maximize sample reuse, a large portion of samples do not effectively contribute to the estimation of each data point's SV. Our analysis and insights could aid the development of future cooperative game theory-based data valuation techniques.

\section{Background}
\label{sec:background}

In this section, we formalize the data valuation problem for ML and review the concept of the Shapley value. 

\paragraph{Data Valuation for Machine Learning.} 
We denote a dataset $D := \{ z_i \}_{i=1}^N$ containing $N$ data points. The objective of data valuation is to assign a score to each training data point in a way that reflects their contribution to model training. Denote $\I := \{1, \ldots, N\}$ as the index set. To analyze a point's ``contribution'', we define a \emph{utility function} $U: 2^N \rightarrow \R$, which maps any subset of the training set to a score indicating the usefulness of the subset. 
$2^N$ represents the power set of $N$, i.e., the collection of all subsets of $N$, including the empty set and $N$ itself. 
For classification tasks, a common choice for $U$ is the validation accuracy of a model trained on the input subset. Formally, for any subset $S \subseteq \I$, we have $U(S) := \metric(\A( \{ z_i \}_{i \in S} ))$, where $\A$ is a learning algorithm that takes a dataset $\{ z_i \}_{i \in S}$ as input and returns a model. $\metric$ is a metric function that evaluates the performance of a given model, e.g., the accuracy of a model on a hold-out test set. 
We consider utility functions with a bounded range, which aligns with the commonly used metric functions such as test classification accuracy. 
Without loss of generality, we assume throughout this note that $U(S) \in [0, 1]$. 
The goal of data valuation is to partition $\Utot := U(\I)$, the utility of the entire dataset, to the individual data point $i \in \I$. 
That is, we want to find a score vector $(\phi_i(U))_{i \in \I}$ where each $\phi_i(U)$ represents the payoff for the owner of the data point $i$.

\paragraph{The Shapley Value.} 
The SV \citep{shapley1953value} is a classic concept in cooperative game theory to attribute the total gains generated by the coalition of all players. 
At a high level, it appraises each point based on the (weighted) average utility change caused by adding the point into different subsets. 
Given a utility function $U(\cdot)$, the Shapley value of a data point $i$ is defined as 
\begin{align}
&\phi_i\left(U\right) := \frac{1}{N} \sum_{k=1}^{N} {N-1 \choose k-1}^{-1} \sum_{S \subseteq \I \setminus \{i\}, |S|=k-1} \left[ U(S \cup \{i\}) - U(S) \right]
\label{eq:shapley-formula}
\end{align}

The popularity of the Shapley value is attributable to the fact that it is the \emph{unique} data value notion satisfying the following four axioms~\citep{shapley1953value}:
\begin{itemize}
    \item Dummy player: if $U\left(S \cup i\right)=U(S)+c$ for all $S \subseteq \I \setminus i$ and some $c \in \R$, then $\phi_i\left(U\right)=c$.
    \item Symmetry: if $U(S \cup i) = U(S \cup j)$ for all $S \subseteq \I \setminus \{i, j\}$, then $\phi_i(U)=\phi_j(U)$. 
    \item Linearity: For any of two utility functions $U_1, U_2$ and any $\alpha_1, \alpha_2 \in \R$, $\phi_i \left( \alpha_{1} U_{1}+\alpha_{2} U_{2}\right)=\alpha_{1} \phi_i\left(U_{1}\right)+$ $\alpha_{2} \phi_i\left( U_{2}\right)$.
    \item Efficiency: for every $U$, $\sum_{i \in \I} \phi_i(U)=U(\I)$.
\end{itemize}
The difference $U(S \cup i) - U(S)$ is often termed the \emph{marginal contribution} of data point $i$ to subset $S \subseteq \I \setminus i$. 
We refer the readers to \citep{ghorbani2019data, jia2019towards} and the references therein for a detailed discussion about the interpretation and necessity of the four axioms in the ML context.

\section{Improved Analysis for the Group Testing-based SV Estimator from \cite{jia2019towards}}

The main obstacle to using the SV is its computational expense. To compute the exact SV as given in Equation (\ref{eq:shapley-formula}), it is necessary to calculate the marginal contribution of each data point $i$ to every possible coalition $S \subseteq \I \setminus {i}$, which is a total of $2^{N-1}$ coalitions. 
To make matters worse, for many ML tasks, each time of evaluating the utility function can already be computationally expensive as it involves training an ML model. 

To meet this challenge, \cite{jia2019towards} propose a Group Testing-based SV estimator which achieves superior asymptotic sample complexity. In this section, we review the estimation algorithm from \cite{jia2019towards}, and we provide an improved sample complexity analysis for it. 

We start by defining the performance metric of a Monte Carlo-based SV estimator. 
\begin{definition}
We say an SV estimator $\widehat \phi \in \R^N$ is an $(\eps, \delta)$-approximation to the true Shapley value $\phi = (\phi_i)_{i \in \I}$ (in $\ell_2$-norm) if and only if
\begin{align*}
\Pr_{\widehat \phi} \left[ \norm{ \phi - \widehat \phi } \le \eps \right] \ge 1-\delta
\end{align*}
where the randomness is taken over the coin flip of the estimator $\widehat \phi$. 
\end{definition}

The first and so far simplest estimation algorithm for SV is called \emph{permutation sampling} \citep{castro2009polynomial}, which samples marginal contributions for each data point $i$ based on randomly drawn permutations of the dataset. Hoeffding's inequality shows that to achieve $(\eps, \delta)$-approximation, permutation sampling estimator requires $\Omega \left( \frac{N^2}{\eps^2} \log(\frac{N}{\delta}) \right)$ evaluations for the utility $U(\cdot)$. 
In the algorithm of permutation sampling, each sampled $U(S)$ is only being used for the Shapley value estimation for at most 2 data points, which results in an extra factor of $N$ in its sample complexity. 

In contrast, the Group Testing-based SV estimator developed in \cite{jia2019towards} requires only $\widetilde \Omega \left( \frac{N}{\eps^2} \right)$ evaluations of $U(\cdot)$.\footnote{We use $\widetilde \Omega$ to hide the logarithmic factors in $\Omega$.} 
The principle behind such an improvement is that the Group Testing-based SV estimator increases the sample reuse, where each sampled $U(S)$ is used for the SV estimation for all of $N$ data points. 
The pseudo-code is summarized in Algorithm \ref{alg:gt-original}. 
At a high level, Group Testing-based estimation algorithm first estimates the \emph{difference} of the Shapley value between every pair of data points $\phi_i-\phi_j$, and then recovers each $\phi_i$ from the value differences. 

By the definition of the Shapley value, one can write the difference $\phi_i-\phi_j$ as follows:

\begin{lemma}
For any $i, j \in \I$, the difference in the Shapley value between $i$ and $j$ is
\begin{align}
\Delta_{i, j} := 
\phi_i-\phi_j = \frac{1}{N-1} \sum_{S \subseteq \I \setminus \{i, j\}} {N-2 \choose |S|}^{-1} [U(S \cup i)-U(S \cup j)]
\label{eq:shap-diff}
\end{align}
\end{lemma}

\begin{algorithm}[t]
\SetAlgoLined
\SetKwInOut{Input}{input}
\SetKwInOut{Output}{output}
\Input{Training set $D = \{(x_i,y_i)\}_{i=1}^N$, utility function $U(\cdot)$, evaluation budget $T$.}
\Output{The estimated SV for each training point $\widehat \phi \in \R^N$.}

$Z \leftarrow 2 \sum_{k=1}^{N-1} \frac{1}{k}$. 

$q_k \leftarrow \frac{1}{Z} (\frac{1}{k} + \frac{1}{N-k})$ for $k=1,\cdots,N-1$. 

Initialize matrix $B \leftarrow \bm{0} \in \R^{T \times N}$. 

\For{$t = 1$ to $T$}{
    Draw $k \sim \{1, \ldots, N-1\}$ according to distribution $q_k$. 
    
    Uniformly sample a size-$k$ subset $S$ from $\I = \{1,\cdots,N\}$. 
    
    $B_{ti} \leftarrow 1$ for all $i \in S$. 
    
    $u_t \leftarrow U(S)$. 
}
$\widehat \Delta_{i, j} \leftarrow \frac{Z}{T} \sum_{t=1}^T u_t (B_{t, i}- B_{t, j})$ for every pair of $(i, j)$ s.t. $i, j \in \I$ and $j > i$. 

Find $\widehat \phi$ by solving the feasibility problem 
\begin{equation}
\begin{array}{rl}
\sum_{i=1}^N \widehat \phi_i &= U(\I), \\
\left| (\widehat \phi_i - \widehat \phi_j) - \widehat \Delta_{i, j} \right| 
&\leq \eps/(2\sqrt{N}),~~~\forall i,j \in \I, j > i. 
\end{array}
\label{eq:feasibility}
\end{equation}

\caption{Group Testing-Based SV Estimation from \cite{jia2019towards}.}
\label{alg:gt-original}
\end{algorithm}


We now provide a more fine-grained analysis for the original Group Testing-based algorithm. The original estimator from \cite{jia2019towards} first derives accurate estimated SV differences $\widehat \Delta_{i, j}$ for all $N(N-1)/2$ pairs of data points $(i, j)$ with high probability ($\ge 1 - \delta$). The correctness of the above algorithm then follows from the lemma here. 

\begin{lemma}[Modified from Lemma 2 in \citet{jia2019towards}]
    If the estimated SV differences $\widehat \Delta_{i, j}$ satisfies $\left| \widehat \Delta_{i, j} - \Delta_{i, j} \right| \le \frac{\eps}{2 \sqrt{N}}$ for all pairs of $(i, j)$ s.t. $i, j \in \I$ and $j > i$, then any solution $\widehat \phi$ for the feasibility problem in (\ref{eq:feasibility}) satisfies $\norm{ \widehat \phi - \phi }_\infty \le \frac{\eps}{\sqrt{N}}$, which further implies $\norm{ \widehat \phi - \phi }_2 \le \eps$. 
    \label{lemma:feasibility}
\end{lemma}

The following theorem provides an improved lower bound on the number of utility evaluations $T$ needed to achieve an $(\eps, \delta)$-approximation.

\begin{theorem}[Improved version of Theorem 3 in \citet{jia2019towards}]
\label{thm:gt-original-comp}
Algorithm \ref{alg:gt-original} returns an $(\eps,\delta)$-approximation to the Shapley value if the number of tests $T$ satisfies 
\begin{align}
    T 
    &\ge \frac{ \log( N(N-1) / \delta ) }{ (1-\qtot) \cdot h \left( \frac{\eps}{2 Z \sqrt{N} (1-\qtot) } \right) } = \Theta \left(\frac{N}{\eps^2} \log(N) \left(\log \frac{N}{\delta}\right) \right) 
    \label{eq:gt-orig-comp}
\end{align}
where $Z := 2\sum _{k=1}^{N-1} \frac{1}{k}$, $q_k := \frac{1}{Z} (\frac{1}{k} + \frac{1}{N-k})$ for $k=1,\cdots,N-1$, $\qtot := \frac{N-2}{N}q_1 + \sum_{k=2}^{N-1} q_k \left(1+\frac{2k(k-N)}{N(N-1)}\right)$, and $h(u) := (1+u)\log(1+u)-u$.
\end{theorem}

Compared with the non-asymptotic sample lower bound in \cite{jia2019towards}, the non-asymptotic lower bound in (\ref{eq:gt-orig-comp}) saves a factor of 8 and improves $\qtot^2$ to $\qtot$. Such an improvement is significant when the non-asymptotic bound is used for providing a confidence interval for the estimated SV. 
The proof (deferred to the Appendix) is also arguably simpler. 


\section{Improved Algorithm Design for Group Testing-based SV Estimator}

In \cite{jia2019towards}, the Group Testing-based SV estimation algorithm requires that with high probability, the estimated $\widehat \Delta_{i, j}$ to be within $\frac{\eps}{2\sqrt{N}}$ error with respect of the true value $\Delta_{i, j}$ \emph{for every pair} of data points $i, j$. There are $N(N-1)/2$ such pairs in a dataset of size $N$, which leads to a factor of $\log(N(N-1)/\delta)$ in the number of required samples. 
Moreover, there might be multiple solutions for the feasibility problem in (\ref{eq:feasibility}), which means that the outputted $\widehat \phi = (\widehat \phi_i)_{i=1}^N$ will depend on the actual solver, which poses difficulties for fine-grained analysis for such an estimator. 

\paragraph{Recovering the Shapley values from $N-1$ pair of differences.} 
In fact, we do \emph{not} need to estimate the Shapley value differences $\widehat \Delta_{i, j}$ for every pair of data points $i, j$. If we already know the Shapley value of a ``pivot'' data point $*$, we can simply estimate $\phi_i - \phi_*$ for every $i \in \I$, and compute $\widehat \phi_i = \phi_* + \widehat \Delta_{i, *}$.\footnote{This technique first appears in Ruoxi Jia's Ph.D. thesis \citep{jia2018accountable}.} 
That is, only $N-1$ pairs of Shapley difference need to be estimated. 
However, in this case, we also need to directly estimate $\phi_*$, the Shapley value of the ``pivot'' data point $*$. 
This requires us to attribute a certain amount of computational budget solely used for estimating $\phi_*$, which introduces an additional step in the estimation algorithm. Moreover, the samples for estimating $\phi_*$ cannot be reused for estimating $\phi_i - \phi_*$.

\begin{algorithm}[t]
\SetAlgoLined
\SetKwInOut{Input}{input}
\SetKwInOut{Output}{output}
\Input{Training set $D = \{(x_i,y_i)\}_{i=1}^N$, utility function $U(\cdot)$, evaluation budget $T$.}
\Output{The estimated SV for each training point $\widehat \phi \in \R^N$.}

$Z \leftarrow 2 \sum_{k=1}^{N} \frac{1}{k}$. 

$q_k \leftarrow \frac{1}{Z} (\frac{1}{k} + \frac{1}{ N +1-k})$ for $k=1,\cdots,N$. 

Initialize matrix $B \leftarrow \bm{0} \in \R^{T \times (N+1)}$. 

\For{$t = 1$ to $T$}{
    Draw $k \sim \{1, \ldots, N\}$ according to distribution $q_k$. 
    
    Uniformly sample a size-$k$ subset $S$ from $\I = \{1,\cdots,N, N+1\}$. 
    
    $B_{ti} \leftarrow 1$ for all $i \in S$. 
    
    $u_t \leftarrow U'(S)$. 
}
$\widehat \Delta_{i, *} \leftarrow \frac{Z}{T} \sum_{t=1}^T u_t (B_{t, i}- B_{t, N+1})$ for all $i \in \{1, \ldots, N\}$. 

$\widehat \phi_i \leftarrow \widehat \Delta_{i, *}$ for all $i \in \{1, \ldots, N\}$. 

\caption{Improved Group Testing-Based SV Estimation.}
\label{alg:gt-new}
\end{algorithm}

\paragraph{Dummy Player technique.} 
Recall that the dummy player axiom of the Shapley value says that for a player $i$, if $U(S \cup i) = U(S)$ for all $S \subseteq \I \setminus i$, then $\phi_i=0$. 
Given a training dataset $\I = \{1, \ldots, N\}$, we augment it by adding a \emph{dummy player} called $*$, i.e., $\I' = \I \cup \{*\}$. For any utility function $U$, we augment it by setting $U'(S)=U(S)$ and $U'(S \cup *) = U(S)$ for all $S \subseteq \I$. Thus, we have $\phi_*(U') = 0$. 
We show that the augmentation from $U$ to $U'$ does not change the Shapley value of any original data points $i \in \I$. 

\begin{theorem}
$\phi_i(U') = \phi_i(U)$ for all $i \in \I$. 
\end{theorem}
\begin{proof}
\begin{align*}
    \phi_i(U') &= \frac{1}{N+1} \sum_{k=0}^N {N \choose k}^{-1} \sum_{S \subseteq \I' \setminus i, |S|=k} U'(S \cup i) - U'(S) \\
    &= \frac{1}{N+1} \sum_{k=0}^{N-1} \sum_{S \subseteq \I' \setminus\{i, *\}, |S|=k} \left[ {N \choose k}^{-1} [U'(S \cup i) - U'(S)] \right. \\
    &~~~~~~~~~~~~~~~~~~~~~~~~~~~~~~~~~~~~~~~\left.+ {N \choose k+1}^{-1} [U'(S \cup \{i, *\}) - U'(S \cup *)] \right] \\
    &= \frac{1}{N+1} \sum_{k=0}^{N-1} \sum_{S \subseteq \I \setminus i, |S|=k} \left({N \choose k}^{-1} + {N \choose k+1}^{-1}\right) [U(S \cup i) - U(S)]  \\
    &= \frac{1}{N+1} \sum_{k=0}^{N-1} \sum_{S \subseteq \I \setminus i, |S|=k} \frac{N+1}{N} {N-1 \choose k}^{-1} [U(S \cup i) - U(S)] \\
    &= \frac{1}{N} \sum_{k=0}^{N-1} \sum_{S \subseteq \I \setminus i, |S|=k} {N-1 \choose k}^{-1} [U(S \cup i) - U(S)] \\
    &= \phi_i(U)
\end{align*}
\end{proof}

\begin{theorem}
\label{thm:gt-original-new}
The improved Group Testing-based SV estimator (Algorithm \ref{alg:gt-new}) returns an $(\eps, \delta)$-approximation to the Shapley value if the number of tests $T$ satisfies 
\begin{align}
    T 
    &\ge \frac{ \log( N / \delta ) }{ (1-\qtot) \cdot h \left( \frac{\eps}{ Z \sqrt{N+1} (1-\qtot) } \right) } = O\left(\frac{N}{\eps^2} \log (N) \left(\log \frac{N}{\delta}\right)\right) 
    \label{eq:gt-new-comp}
\end{align}
where $Z := 2\sum _{k=1}^{N} \frac{1}{k}$, $q_k := \frac{1}{Z} (\frac{1}{k} + \frac{1}{N+1-k})$ for $k=1,\cdots,N$, $\qtot := \frac{N-1}{N+1} q_1 + \sum_{k=2}^{N} q_k \left(1+\frac{2k(k-N-1)}{N(N+1)}\right)$, and $h(u) := (1+u)\log(1+u)-u$.
\end{theorem}

While the sample complexity of the modified group testing estimator is still $\tildeO(N)$ asymptotically, it improves the non-asymptotic sample bound by a factor of 2 since we only need to estimate $N$ pairs instead of $N(N-1)/2$ pairs of Shapley differences. 

\begin{remark}
Algorithm \ref{alg:gt-new} actually uses an alternative expression for the Shapley value
\begin{align*}
    \phi_i 
    &= Z \sum_{k=0}^{N-1} \frac{q_{k+1}}{ {N+1 \choose k+1} } \sum_{S \subseteq \I \setminus \{i\}, |S|=k} [U(S \cup i) - U(S)] \\
    &= Z \E_{k \sim q_k, S \sim \unif(\I \cup \{*\}), |S|=k}\left[ U(S)\ind[i \in S, * \notin S] - U(S)\ind[i \notin S, * \in S] \right] 
\end{align*}
\end{remark}


\section{Sample Reuse in Group Testing-based SV estimator}

As mentioned earlier, the Group Testing-based SV estimator saves a factor of $N$ in sample complexity compared with the permutation sampling estimator by increasing the sample reuse, i.e., each sampled utility $U(S)$ is used in the estimation of all $\Delta_{i, j}$'s, and hence the estimation of all $\phi_i$'s. 
However, this gives a false sense that all information is being utilized. 
Algorithm \ref{alg:gt-original} or \ref{alg:gt-new} constructs the sampling distribution of $S$ such that $\Delta_{i, j} = \E_{S} \left[ Z (\ind[i \in S] - \ind[j \in S]) U(S) \right]$. 
This means that among all of the sampled $S$, those who include or exclude both $i, j$ have zero contributions to the estimation of $\Delta_{i, j}$. 
Moreover, it can be shown that the probability of sampling such ineffective $S$ is 
\begin{align*}
    \qtot := \Pr[i, j \in S \text{ or } i, j \notin S] = 1 - \frac{2}{Z}
\end{align*}
(see the detailed derivation in the Appendix). 
A simple bound for the sum of harmonic series tells us that $Z = \Theta(\log N)$. Hence, we know that the ``effective'' amount of samples that can be used to estimate $\Delta_{i, j}$ is only $\Theta(1 / \log N)$ out of all collected samples. 
Therefore, the Group Testing-based estimator does not achieve the maximal possible sample reuse where all of the sampled $U(S)$ is being used to estimate all $\phi_i$'s. 

Is it possible to construct a Monte-Carlo-based SV estimator that achieves maximal sample reuse? It has been shown in \cite{wang2022data} that it is impossible to construct a distribution $\D$ such that $\phi_{i}\left(U\right) = \E_{S \sim \D|\D \own i} \left[ U(S) \right] - \E_{S \sim \D|\D \notown i} \left[ U(S) \right]$ for all $i \in \I$. Hence, an SV estimator that achieves maximal sample reuse cannot simply decide how to use the sampled $U(S)$ only based on the membership of a data point $i$. That is, it is impossible to find a single distribution $\D$ over $S$ such that $\widehat \phi$ is in the form of 
\begin{align}
    \widehat \phi_i = \frac{1}{|\Sowni|} \sum_{S \in \Sowni} U(S) - \frac{1}{|\Snotowni|} \sum_{S \in \Snotowni} U(S) \label{eq:group-testing}
\end{align}
where $\mS = \{ S_1, \ldots, S_m \}$ i.i.d. drawn from $\D$, $\Sowni = \{S \in \mS: i \in S \}$ and $\Snotowni = \{S \in \mS: i \notin S\} = \mS \setminus \Sowni$. 
The Group Testing-based SV estimator in \cite{jia2019towards} makes use of the membership of other data points in $S$ and improves the extent of sample reuse, which can serve as an excellent starting point for future research. 

\begin{remark}
While the Group Testing-based SV estimator achieves more favorable asymptotic sample complexity compared with the  permutation sampling estimator, it has been reported in several subsequent data valuation studies that the actual performance of the Group Testing-based estimator does not observably outperform permutation sampling technique \citep{wang2020principled, yan2020ifyoulike, wang2021improving, wang2022data}. In particular, \cite{wang2022data} also performed the empirical evaluation for the improved Group Testing-based estimator proposed in this note, and still do not observe significant advantages in SV estimation accuracy. This is because, despite the low sample reuse, the permutation sampling estimator directly samples the marginal contribution, which significantly reduces the estimation variance when $N$ is small. 
\end{remark}

\section{Conclusion}

In this technical note, we presented several improvements to the analysis and algorithm design of the Group Testing-based Shapley value estimator. 
Our work contributes to a better understanding of the challenges associated with developing Monte Carlo-based SV estimators. We hope our insights can inspire future research on data valuation techniques.

\newpage

\bibliographystyle{plainnat}
\bibliography{ref}

\begin{thebibliography}{13}
\providecommand{\natexlab}[1]{#1}
\providecommand{\url}[1]{\texttt{#1}}
\expandafter\ifx\csname urlstyle\endcsname\relax
  \providecommand{\doi}[1]{doi: #1}\else
  \providecommand{\doi}{doi: \begingroup \urlstyle{rm}\Url}\fi

\bibitem[Castro et~al.(2009)Castro, G{\'o}mez, and
  Tejada]{castro2009polynomial}
Javier Castro, Daniel G{\'o}mez, and Juan Tejada.
\newblock Polynomial calculation of the shapley value based on sampling.
\newblock \emph{Computers \& Operations Research}, 36\penalty0 (5):\penalty0
  1726--1730, 2009.

\bibitem[Ghorbani and Zou(2019)]{ghorbani2019data}
Amirata Ghorbani and James Zou.
\newblock Data shapley: Equitable valuation of data for machine learning.
\newblock In \emph{International Conference on Machine Learning}, pages
  2242--2251. PMLR, 2019.

\bibitem[Jia(2018)]{jia2018accountable}
Ruoxi Jia.
\newblock \emph{Accountable Data Fusion and Privacy Preservation Techniques in
  Cyber-Physical Systems}.
\newblock University of California, Berkeley, 2018.

\bibitem[Jia et~al.(2019)Jia, Dao, Wang, Hubis, Hynes, G{\"u}rel, Li, Zhang,
  Song, and Spanos]{jia2019towards}
Ruoxi Jia, David Dao, Boxin Wang, Frances~Ann Hubis, Nick Hynes, Nezihe~Merve
  G{\"u}rel, Bo~Li, Ce~Zhang, Dawn Song, and Costas~J Spanos.
\newblock Towards efficient data valuation based on the shapley value.
\newblock In \emph{The 22nd International Conference on Artificial Intelligence
  and Statistics}, pages 1167--1176. PMLR, 2019.

\bibitem[Koh and Liang(2017)]{koh2017understanding}
Pang~Wei Koh and Percy Liang.
\newblock Understanding black-box predictions via influence functions.
\newblock In \emph{International Conference on Machine Learning}, pages
  1885--1894. PMLR, 2017.

\bibitem[Kwon and Zou(2021)]{kwon2021beta}
Yongchan Kwon and James Zou.
\newblock Beta shapley: a unified and noise-reduced data valuation framework
  for machine learning.
\newblock \emph{arXiv preprint arXiv:2110.14049}, 2021.

\bibitem[Shapley(1953)]{shapley1953value}
Lloyd~S Shapley.
\newblock A value for n-person games.
\newblock \emph{Contributions to the Theory of Games}, 2\penalty0
  (28):\penalty0 307--317, 1953.

\bibitem[Tian et~al.(2022)Tian, Ding, Fu, and Liu]{tian2022data}
Yingjie Tian, Yurong Ding, Saiji Fu, and Dalian Liu.
\newblock Data boundary and data pricing based on the shapley value.
\newblock \emph{IEEE Access}, 10:\penalty0 14288--14300, 2022.

\bibitem[Wang and Jia(2023)]{wang2022data}
Jiachen~T. Wang and Ruoxi Jia.
\newblock Data banzhaf: A robust data valuation framework for machine learning.
\newblock \emph{International Conference on Artificial Intelligence and
  Statistics}, 2023.

\bibitem[Wang et~al.(2020)Wang, Rausch, Zhang, Jia, and
  Song]{wang2020principled}
Tianhao Wang, Johannes Rausch, Ce~Zhang, Ruoxi Jia, and Dawn Song.
\newblock A principled approach to data valuation for federated learning.
\newblock In \emph{Federated Learning}, pages 153--167. Springer, 2020.

\bibitem[Wang et~al.(2021)Wang, Yang, and Jia]{wang2021improving}
Tianhao Wang, Yu~Yang, and Ruoxi Jia.
\newblock Improving cooperative game theory-based data valuation via data
  utility learning.
\newblock \emph{arXiv preprint arXiv:2107.06336}, 2021.

\bibitem[Yan and Procaccia(2020)]{yan2020ifyoulike}
Tom Yan and Ariel~D Procaccia.
\newblock If you like shapley then you’ll love the core, 2020.

\bibitem[Zhu et~al.(2019)Zhu, Dong, Shen, and Gai]{zhu2019incentive}
Liehuang Zhu, Hui Dong, Meng Shen, and Keke Gai.
\newblock An incentive mechanism using shapley value for blockchain-based
  medical data sharing.
\newblock In \emph{2019 IEEE 5th Intl Conference on Big Data Security on Cloud
  (BigDataSecurity), IEEE Intl Conference on High Performance and Smart
  Computing,(HPSC) and IEEE Intl Conference on Intelligent Data and Security
  (IDS)}, pages 113--118. IEEE, 2019.

\end{thebibliography}


\newpage
\onecolumn

\appendix

\section{Proofs}

\subsection{Proof for Lemma \ref{lemma:feasibility}}

\begin{customthm}{\ref{lemma:feasibility}}
If the estimated SV differences $\widehat \Delta_{i, j}$ satisfies $\left| \widehat \Delta_{i, j} - \Delta_{i, j} \right| \le \frac{\eps}{2 \sqrt{N}}$ for all pairs of $(i, j)$ s.t. $i, j \in \I$ and $j > i$, then any solution $\widehat \phi$ for the feasibility problem (\ref{eq:feasibility}) satisfies $\norm{ \widehat \phi - \phi }_\infty \le \frac{\eps}{\sqrt{N}}$, which further implies $\norm{ \widehat \phi - \phi }_2 \le \eps$. 
\end{customthm}
\begin{proof}
Suppose, for contradiction, that there exists $i \in \I$ such that $\left| \widehat \phi_i - \phi_i \right| > \frac{\eps}{\sqrt{N}}$. Denote $\eps' := \eps/(2\sqrt{N})$. 
By assumption, for arbitrary $j \ne i$, we have 
\begin{align}
\left| \widehat \Delta_{i, j} - \Delta_{i, j} \right| 
= \left| (\phi_i-\phi_j) - \widehat \Delta_{i, j} \right| \leq \eps' \nonumber
\end{align}
The constraint $\left| (\widehat \phi_i - \widehat \phi_j) - \widehat \Delta_{i, j} \right| \le \eps'$ implies that 
\begin{align*}
    |(\widehat \phi_i - \phi_i) - (\widehat \phi_j - \phi_j)| \leq |\widehat \phi_i - \widehat \phi_j - \widehat \Delta_{i, j}| + |\phi_i - \phi_j - \widehat \Delta_{i, j}| \le 2\eps'
\end{align*}
Denote $\widehat \phi_i-\phi_i = c \eps'$. 
We have
\begin{align*}
    (c-2)\eps' \leq \widehat \phi_j - \phi_j\leq (c+2) \eps'
\end{align*}

The assumption of $\left| \widehat \phi_i - \phi_i \right| > \frac{\eps}{\sqrt{N}} = 2 \eps'$ either implies $c > 2$ or $c < -2$. 
If $c > 2$, we have $\widehat \phi_j - \phi_j > 0$ for any $j \ne i$. 
Then,
\begin{align*}
    \sum_{j=1}^N (\widehat \phi_j - \phi_j) = \sum_{j\neq i} (\widehat \phi_j - \phi_j) + (\widehat \phi_i - \phi_i) > 0
\end{align*}

By efficiency axiom, the sum of the Shapley value $\sum_{j=1}^N \phi_j = U(\I)$, it follows that $\sum_{j=1}^N \widehat \phi_j > U(\I)$, which does not satisfy all of the constraints in (\ref{eq:feasibility}). A similar contradiction can be made for the case of $c < -2$. Hence, for all $i \in \I$, we have $\left| \widehat \phi_i - \phi_i \right| \le \frac{\eps}{\sqrt{N}}$. 
\end{proof}

\newpage

\subsection{Proof for Theorem \ref{thm:gt-original-comp}}

To prove Theorem \ref{thm:gt-original-comp}, we use Bennett's inequality. 

\begin{lemma}[Bennett’s inequality]
\label{thm:bennett}
Let $X_1, \ldots, X_n$ be independent real-valued random variables with finite variance such that $X_i \le b$ for some $b>0$ almost surely for all $1 \le i \le n$. 
Let $\nu \ge \sum_{i=1}^n \E[X_i^2]$. For any $t > 0$, we have 
\begin{align}
    \Pr \left[ \sum_{i=1}^n X_i - \E[X_i] \ge t \right] 
    \le \exp \left( -\frac{\nu}{b^2} h \left(\frac{bt}{\nu}\right) \right)
\end{align}
where $h(x) = (1+x)\log(1+x)-x$ for $x > 0$. 
\end{lemma}

\begin{customthm}{\ref{thm:gt-original-comp}}
Algorithm \ref{alg:gt-original} returns an $(\eps,\delta)$-approximation to the Shapley value if the number of tests $T$ satisfies 
\begin{align}
    T &\ge \frac{ \log( N(N-1) / \delta ) }{ (1-\qtot) \cdot h \left( \frac{\eps}{2 Z \sqrt{N} (1-\qtot) } \right) } = O\left(\frac{n}{\eps^2} \left(\log \frac{n}{\delta}\right)^2\right) 
\end{align}
where $Z := 2\sum _{k=1}^{N-1} \frac{1}{k}$, $q_k := \frac{1}{Z} (\frac{1}{k} + \frac{1}{N-k})$ for $k=1,\cdots,N-1$, $\qtot := \frac{N-2}{N}q_1 + \sum_{k=2}^{N-1} q_k \left(1+\frac{2k(k-N)}{N(N-1)}\right)$, and $h(u) := (1+u)\log(1+u)-u$.
\end{customthm}
\begin{proof}

For a subset $S \subseteq \I$, we use $\beta = (\beta_1, \ldots, \beta_N)$ to denote its binary representation where $\beta_i := \ind[i \in S]$. 
In Algorithm \ref{alg:gt-original}, the subset $S$ are sampled as follows: 

\begin{enumerate}
\item Sample subset size $k \in \{1, 2,\cdots,N-1\}$ according to the discrete distribution $(q_1, \ldots, q_{k-1})$. 
\item Uniformly sample a size-$k$ subset $S \subseteq \I$. 
\end{enumerate}

Consider any two data points $i, j \in \I$. 
In Algorithm \ref{alg:gt-original}, the Shapley value difference $\Delta_{i, j}$ is estimated based on samples from the distribution of $\zeta := Z (\beta_i - \beta_j) U(S)$. We first verify that such an estimate is unbiased. Obviously, only $\beta_i \ne \beta_j$ has non-zero contributions to its expectation. 

\begin{align}
    \E [\zeta] = \E[ Z (\beta_i - \beta_j) U(S) ] &= 
    Z \sum_{k=0}^{N-2} \frac{q_{k+1}}{{N \choose k+1}} 
    \sum_{S \subseteq \I \setminus \{i, j\}, |S|=k} \left[ U(S \cup i) - U(S \cup j) \right] 
\end{align}

Note that 
\begin{align*}
    Z\frac{q_{k+1}}{ {N \choose k+1} } 
    &=
    \left(\frac{1}{k+1} + \frac{1}{N-k-1}\right) \frac{1}{ {N \choose k+1}} \\ 
    &= \frac{N}{(k+1)(N-k-1)} \frac{(k+1)! (N-k-1)!}{N!} \\
    &= \frac{ k! (N-k-2)!}{(N-1)!} \\
    &= \frac{1}{(N-1) {N-2 \choose k} }
\end{align*}
which leads to 

\begin{align*}
    \E[ \zeta ] &= 
    \frac{1}{N-1} \sum_{k=0}^{N-2} \frac{1}{ {N-2 \choose k} }
    \sum_{S \subseteq \I \setminus \{i, j\}, |S|=k} \left[ U(S \cup i) - U(S \cup j) \right]  \\
    &= \phi_i - \phi_j
\end{align*}

Now, we analyze the variance of $\zeta$, in order to apply Bennett's inequality to bound the sample complexity for estimating $\Delta_{i, j}$. 
Note that $U \in [0, 1]$, which means that $\zeta \in [-1, 1]$. 
Recall that $\zeta = 0$ when $\beta_i = \beta_j$. 
The probability of such event is 
\begin{align*}
    \qtot := \Pr[\zeta = 0] = \Pr[ \beta_i = \beta_j ] 
    &= \Pr[\beta_i=1, \beta_j=1] + \Pr[\beta_i=0, \beta_j=0]\\
    &= \sum_{k=2}^{N-1} \frac{q_k}{{N \choose k}} {N-2 \choose k-2} + \sum_{k=1}^{N-1} \frac{q_k}{{N \choose k}} {N-2 \choose k} \\
    &= \frac{N-2}{N} q_1 + \sum_{k=2}^{N-1} \frac{q_k}{{N \choose k}} \left( {N-2 \choose k-2} + {N-2 \choose k} \right) \\
    &= \frac{N-2}{N} q_1 + \sum_{k=2}^{N-1} q_k \left( 1 + \frac{2k(k-N)}{N(N-1)} \right)
\end{align*}

Hence, $\E[\zeta^2]$ can be bounded as follows:
\begin{align*}
    \E[\zeta^2] 
    &= \Pr[\zeta \ne 0] \E[\zeta^2 | \zeta \ne 0] \\
    &\le 1 - \qtot 
\end{align*}


Given $T$ i.i.d. samples $(\zeta_t)_{t=1}^T$, by Bennett's inequality, we have 
\begin{align*}
    \Pr \left[ \left| \widehat \Delta_{i, j} - \Delta_{i, j} \right| \ge \eps \right] 
    = 
    \Pr \left[ \left| \frac{1}{T} \sum_{t=1}^T \zeta_t - \Delta_{i, j} \right| \ge \eps \right] 
    \le 
    2 \exp \left( -T (1-\qtot) h \left( \frac{\eps}{Z (1-\qtot)} \right) \right)
\end{align*}

Hence, 
\begin{align*}
    &\Pr \left[ \forall i,j~\text{s.t.}~i > j,~\left| \widehat \Delta_{i, j} - \Delta_{i, j} \right| \le \frac{\eps}{2 \sqrt{N}}  \right] \\
    &= 1 - \Pr \left[ \exists i,j~\text{s.t.}~i > j,~\left| \widehat \Delta_{i, j} - \Delta_{i, j} \right| \ge \frac{\eps}{2 \sqrt{N}}  \right] \\
    &\ge 
    1 - N(N-1) \exp \left( -T (1-\qtot) h \left( \frac{\eps}{Z (1-\qtot)} \right) \right)
\end{align*}

By setting the failure probability $\le \delta$, we have that when $T \ge \frac{ \log( N(N-1) / \delta ) }{ (1-\qtot) \cdot h \left( \frac{\eps}{2 Z \sqrt{N} (1-\qtot) } \right) }$, with probability $\ge 1-\delta$ we have 
\begin{align*}
    \left| \widehat \Delta_{i, j} - \Delta_{i, j} \right| \le \frac{\eps}{2\sqrt{N}}
\end{align*}
for all pair of $i, j$ s.t. $i > j$. 
By Lemma \ref{lemma:feasibility}, any solutions $\widehat \phi$ for the feasibility problem in (\ref{eq:feasibility}) under this event has $\norm{\widehat \phi - \phi} \le \eps$, which leads to the conclusion. 

\paragraph{Asymptotic Analysis.} It can be shown that $\qtot = 1 - \frac{2}{Z}$, and 
\begin{align}
    Z(1 - \qtot) = 2
\end{align}
Therefore, as $N\rightarrow \infty$,
\begin{align}
    \frac{\eps}{Z\sqrt{N}(1-\qtot)}\rightarrow 0
\end{align}
The first-order Taylor expansion of $h(u)$ at $0$ is $\frac{u^2}{2}$. Thus, we have
\begin{align}
    \frac{ \log( N(N-1) / \delta ) }{ (1-\qtot) \cdot h \left( \frac{\eps}{2 Z \sqrt{N} (1-\qtot) } \right) }
    & = O(NZ \log N)
\end{align}
Since $Z \leq 2(\log (N-1) + 1)$, we have $O(NZ \log N) = O(N(\log N)^2)$.

\end{proof}

\end{document}